%% file: main_arxiv.tex
\documentclass{article}
\pdfoutput=1

 \PassOptionsToPackage{numbers, compress}{natbib}

\usepackage[preprint]{neurips_2023}
\usepackage{subfiles}
\usepackage{shortcuts}

\usepackage[utf8]{inputenc} 
\usepackage[T1]{fontenc}    
\usepackage{hyperref}       
\usepackage{url}            
\usepackage{booktabs}       
\usepackage{amsfonts}       
\usepackage{nicefrac}       
\usepackage{microtype}      
\usepackage{xcolor}         
\usepackage{amsmath,amssymb}
\usepackage{algorithm}
\usepackage{algorithmic}
\usepackage{graphicx}
\usepackage{multicol}
\usepackage{tikz}
\usepackage{subcaption}
\usepackage{bbm}

\title{Optimally Teaching a Linear Behavior Cloning Agent}

\author{%
  Shubham Kumar Bharti\\
  UW-Madison\\
  \texttt{skbharti@cs.wisc.edu} \\
  \And
  Stephen Wright \\
  UW-Madison\\
  \texttt{swright@cs.wisc.edu} \\
  \And
  Adish Singla \\
  MPI-SWS\\
  \texttt{adishs@mpi-sws.org} \\
 \And
 Xiaojin Zhu \\
 UW-Madison\\
 \texttt{jerryzhu@cs.wisc.edu} \\
}

\begin{document}

\maketitle

\input{abstract}

\input{introduction}

\input{problem_definition_learner}

\input{problem_definition_teacher}
\input{teaching_algorithm}
\input{experiments}
\input{related_works}
\input{conclusion}

\bibliography{main_arxiv}
\bibliographystyle{plain}

\input{appendix}


\end{document}

%% file: abstract.tex
\begin{abstract}
We study optimal teaching of Linear Behavior Cloning (LBC) learners. In this setup, the teacher can select which states to demonstrate to an LBC learner. The learner maintains a version space of infinite linear hypotheses consistent with the demonstration.
The goal of the teacher is to teach a realizable target policy to the learner using minimum number of state demonstrations. This number is known as the Teaching Dimension(TD). We present a teaching algorithm called ``Teach using Iterative Elimination(TIE)'' that achieves  instance optimal TD. However, we also show that finding optimal teaching set computationally is NP-hard.
We further provide an approximation algorithm that guarantees an approximation ratio of $\log(|\cA|-1)$ on the teaching dimension.
Finally, we provide experimental results to validate the efficiency and effectiveness of our algorithm.
\end{abstract}

%% file: introduction.tex
\section{Motivation}
Behavior cloning is a form of imitation learning  \cite{Codevilla_2019_ICCV,DBLP:journals/corr/abs-1905-13566} where a teacher demonstrates selected (state, target action) pairs to a learner.
Naive Behavior Cloning by demonstrating on all states can be very inefficient in environments with a large state space.
However, the learner, utilizing the inductive bias of its hypothesis (policy) family, has the ability to generalize the demonstration to states not demonstrated on, as long as the teacher provides a consistent and sufficient demonstration.
This promotes the question: \textbf{What is the minimum set of demonstration states that can teach a full target policy to the learner?}
This is known as the optimal teaching problem in the context of behavior cloning, and the size of the minimum set is known as the teaching dimension.
Prior works have studied optimal teaching on behavior cloning with finite hypothesis class~\cite{goldman1992complexity}.
This paper takes a significant step forward to allow teaching with an uncountable hypothesis family, specifically the family of linear policies. 

\begin{figure}[!htb]
	\centering
	\includegraphics[width=0.4\textwidth]{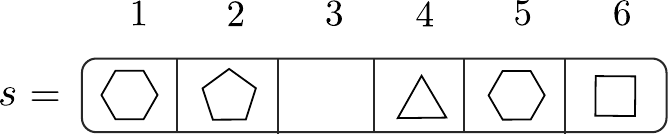}
	\caption{A board in a ``pick the right diamond'' game.  For this board, the target policy defined in Example~\ref{exp:pick_the_diamond_in_intro} says to pick the diamond in slot 5.}
	\label{fig:pick_the_diamond_in_intro}
\end{figure}

\begin{example}[Pick the right diamond]
\label{exp:pick_the_diamond_in_intro} 
The game is shown in Figure~\ref{fig:pick_the_diamond_in_intro}.
There is a board with $n=6$ slots.
Each slot can have one of 4 different diamonds (differ by the number of edges) or be empty, but the whole board is not empty. 
The rule says that one must pick the most expensive diamond on the board, i.e. one with the most edges;
and if there are multiple slots with the most expensive diamond, pick the right-most one.
The teacher wants to demonstrate this rule to the learner.

The state space contains $5^n-1$ states.
The action space is $\{1, \ldots, n\}$, namely, from which slot should the learner choose.
The rule is a target policy $\pi^\dagger$: for any given state (board) $s$, $a=\pi^\dagger(s) \in [n]$ is the correct action at $s$.
The naive behavior cloning demonstration will show the target action on all $5^n-1$ states.
Now suppose that the leaner uses a simple and intuitive linear hypothesis family, where the feature vector consists of the slot index and the number of edges in that slot.
It turns out that a clever teacher can teach the full rule by demonstrating the target action in only two states!
Details are provided in Section~\ref{exp:teach_to_pick_the_diamond}.
\end{example}

\textbf{Our Contributions}: To our knowledge, we are the first to formalize the optimal teaching problem for Linear Behavior Cloning (LBC) learners.
We give a novel algorithm that computes a teaching set that achieves the teaching dimension.
We also give a
computationally efficient approximation with a $\log(|\cA|-1)$ approximation ratio on teaching dimension.

%% file: problem_definition_learner.tex
\section{Problem Formulation}
An LBC teaching problem has two entities: 
a Linear Behavior Cloning leaner and a teacher. Their interaction is defined in Procedure~\ref{alg:interaction}.
\floatname{algorithm}{Procedure} 
\begin{algorithm}[H]
	\caption{LBC Interaction Protocol}\label{alg:interaction}
	\begin{flushleft}
		\textbf{Entities:} a teacher with a target policy $\pi^\dagger$, and an LBC learner with feature function $\phi$
	\end{flushleft}
	\begin{algorithmic}[1]
		\STATE  The teacher constructs a demonstration dataset ${D =\{(s,\pi^\dagger(s)) : s \in T \subseteq \cS\}}$ and sends it to the learner.
		\STATE  During learning, the learner receives $D$ and forms a version space $\VS(D)$ given by~\eqref{eqn:polyhedral_cone}.
		\STATE  During deployment, the learner arbitrarily chooses a policy $\pi_w$ from $\VS(D)$ as defined by~\eqref{eqn:policy_defn}. 
	\end{algorithmic} 
\end{algorithm}

\subsection{The Linear Behavior Cloning Learner}
Let $\cS$ be a finite state space and $\cA$ a finite action space.
The learner uses a feature function $\phi : \cS \times \cA \rightarrow \bR^d$. 
The learner maintains a version space of hypotheses, i.e. a set of weight vectors $w \in \bR^d$.
Each weight vector $w$ induces a set of linear policies $\pi_w$: 

\begin{equation}
\pi_w(s) \in \Delta\left({\arg \max_{a \in \cA} \langle w, \phi(s,a)\rangle}\right),\, \forall s \in \cS \label{eqn:policy_defn}
\end{equation}
where $\Delta(Z)$ denotes the probability simplex over set $Z$. 
If there are no ties in $\arg \max$, $\pi_w(s)$ is simply the best action at state $s$ according to $w$.
If there are ties, the learner chooses an arbitrary mixing (deterministic or stochastic) among the best actions.

The learner accepts a demonstration set $D = \{(s, a): s \in \cS, a \in \cA\}$.
Each item $(s,a)$ in the set means that the learner should strictly prefer action $a$ over all other actions in state $s$. 
Given a demonstration set $D$, the learner maintains a version space $\VS(D)$ of weights consistent with demonstration
\begin{align}
	\VS(D) &= \{w \in \bR^d : \pi_w(s) = a,\, \forall (s,a) \in D\}\\
	= \{&w\in \bR^d : \langle w, \psi_{sab}\rangle > 0,\, \forall (s,a) \in D, b\neq a\}. \label{eqn:polyhedral_cone}
\end{align}
where $\psi_{sab}:=\phi(s,a) -\phi(s,b)$ is the feature difference vector induced for strictly preferring action $a$ over action $b$ in state $s$.
We note that $\VS(D)$ is a polyhedral cone with open faces due to the strict inequalities.
For notational simplicity, we define $\Psi(D) = \{\psi_{sab} : (s,a) \in D, b \in \cA, b\neq a\}$ as the set of all feature difference vectors induced by $D$. Further, the primal cone $\Psi(D)$ (induced by dataset $D$) is defined as $\cP(\Psi(D)) := \{ \sum_{\psi \in \Psi(D)} \lambda_\psi \psi : \lambda_\psi \geq 0\}$, and its dual cone is defined as $\cD(\Psi(D)) := \{w \in \bR^d : \langle w, \psi \rangle > 0, \forall \psi \in \Psi(D), \psi \neq 0\}$. We note that by \eqref{eqn:polyhedral_cone}, $\VS(D)$ is the dual cone of $\Psi(D)$. 

\begin{figure}
	\centering
	\begin{subfigure}[b]{0.235\textwidth}
		\includegraphics[width=\textwidth]{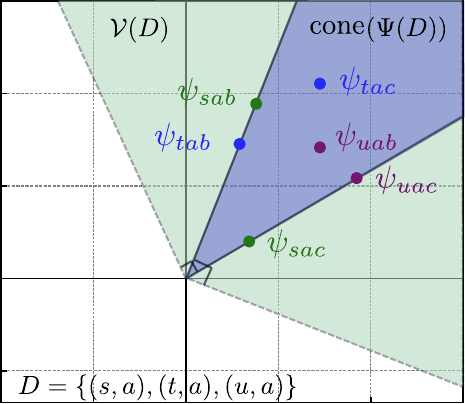}
		\caption{}
		\label{fig:primal_dual_cone}
	\end{subfigure}
	\begin{subfigure}[b]{0.235\textwidth}
		\includegraphics[width=\textwidth]{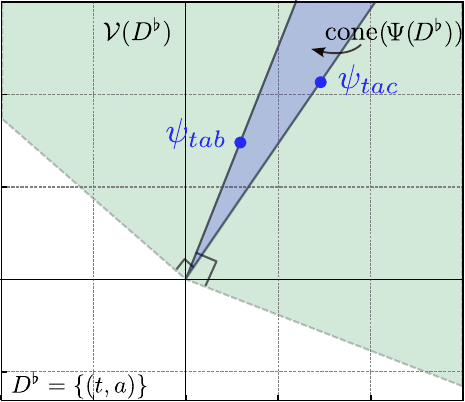}
		\caption{}
		\label{fig:non_teaching}
	\end{subfigure}
	\begin{subfigure}[b]{0.235\textwidth}
		\includegraphics[width=\textwidth]{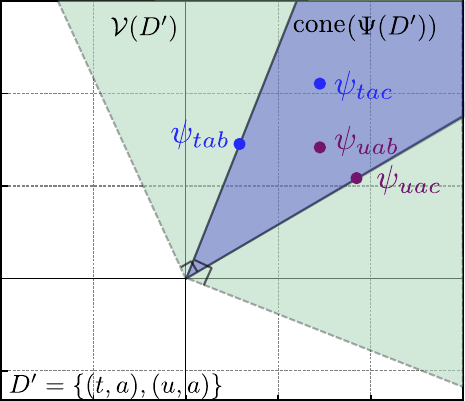}
		\caption{}
		\label{fig:suboptimal_teaching}
	\end{subfigure}
	\begin{subfigure}[b]{0.235\textwidth}
		\includegraphics[width=\textwidth]{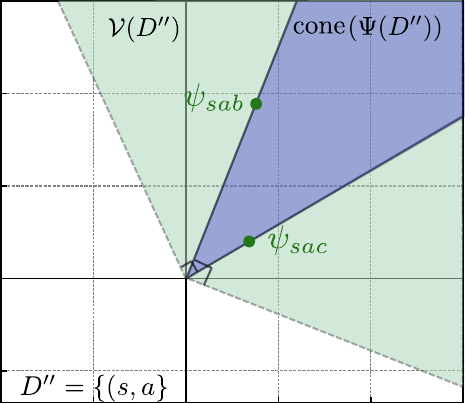}
		\caption{}
		\label{fig:optimal_teaching}
	\end{subfigure}
\caption{The importance of extreme rays.
$D, D', D"$  succeed but $D^\flat$ fails depending on if they cover the extreme rays of $\cP(\Psi(D))$.}
	\label{fig:teaching_diagram}
\end{figure}

\begin{example}[A simple LBC Learner in $\mathbb R^2$]
Let $\mathcal S = \{s, t, u\}$, $\mathcal A = \{a, b, c\}$, and the six feature difference vectors are indicated as dots in Figure~\ref{fig:primal_dual_cone}. 
Consider the full demonstration set $D=\{(s, a), (t, a),(u, a) \}$.
Then $\Psi(D)=\{\psi_{sab}, \psi_{sac}, \psi_{tab}, \psi_{tac}, \psi_{uab}, \psi_{uac}\}$,
the primal cone $\cP(\Psi(D))$ is shown in blue, and the version space $\VS(D)$ in green.  We note that the primal cone is supported by two extreme rays. 
The smaller demonstration set $D^\flat=\{(t, a)\}$ in Figure~\ref{fig:non_teaching} fails
since its primal cone is smaller and the corresponding dual version space is larger than $\VS(D)$, which contains $w$'s inconsistent with $D$. 
But note that two other smaller demonstration sets $D'=\{(t, a),(u, a) \}$ and $D''=\{(s,a)\}$ succeed,
since they each cover both extreme rays of $\cP(\Psi(D_\cS))$
(Figures~\ref{fig:suboptimal_teaching}~and~\ref{fig:optimal_teaching}).

\label{exp:2d_example}
\end{example}

After behavior cloning from the demonstration set $D$, the learner will be deployed.  During deployment, it arbitrarily picks a consistent hypothesis $w\in \VS(D)$ and then chooses a policy $\pi_w$ according to equation~\eqref{eqn:policy_defn}.

%% file: problem_definition_teacher.tex
\subsection{The Teacher}
In our setup, a teacher wants to teach a deterministic target policy ${\pi^\dagger: \mathcal S \rightarrow \mathcal A}$ to the LBC learner.
We use $\pi^\dagger$ to make it clear that the teacher's target policy does not need to be optimal $\pi^*$ with respect to the underlying Markov Decision Process.
The teacher knows everything about the learner, including the feature function $\phi$.
The teacher controls the demonstration dataset $D$ provided to the learner.
\textbf{The teacher wants the learner to {unambiguously} learn the target policy $\pi^\dagger$ using as few demonstrations as possible.}
Unambiguity means that the teacher's demonstration set $D$ should be such that the learner's resulting version space $\VS(D)$ should contain only weights that strictly prefer action $\pi^\dagger(s)$ at each state $s$; there should be no ties otherwise the learner could have chosen an alternative action.
Minimum demonstration means the cardinality $|D|$ should be minimized. This is possible because the learner uses linear policies; demonstrating the target action at one state can constrain the learner's action at other states.
The minimum cardinality $|D|$ is known as the teaching dimension and is defined below.

Formally, an LBC teaching problem instance is defined by the tuple $(\cS, \cA, \phi, \pi^\dagger)$.

\begin{assumption}[Realizability]
The target policy $\pi^\dagger$ is realizable under the linear feature mapping $\phi$, i.e.
	\begin{align*}
		\exists w \in \mathbb R^d \text{ s.t. }\, \forall s \in \mathcal S, \{\pi^\dagger(s)\} = \underset{a \in \mathcal A}{\arg \max} \langle w, \phi(s,a)\rangle.
	\end{align*}
 \label{assum:realizability}
\end{assumption}
The realizability assumption is common in the optimal teaching literature~\cite{goldman1993teaching}.  We note that under realizability, $\psi_{s,\pi^\dagger(s),a} \neq 0,\, \forall a \neq \pi^\dagger(s)$.
Let 
\begin{equation}
{D_\cS=\{(s,\pi^\dagger(s)):s\in\cS\}}
\end{equation}
be the full demonstration set. 
Under realizability, $\VS(D_\cS)$ is a non-empty open cone in $\bR^d$. 
The teacher can certainly use $D_\cS$ to successfully teach the target policy $\pi^\dagger$ to the learner.
However, in general, $D_\cS$ is not the minimum demonstration set for $\pi^\dagger$. 
Practically speaking, many real-world applications have large state space $\cS$, and it will be undesirable to use the full demonstration set $D_\cS$.

To find a minimum demonstration set, the teacher solves the following optimal teaching problem:
\begin{eqnarray}
\min_{T \subseteq \cS} && |T| \label{eq:cardT} \label{eqn:optimal_teach} \\
\mbox{s.t.} && D_T = \{(t,\pi^\dagger(t)) : t \in T\} \label{eq:Dstar} \\
            && \forall w \in \VS(D_T), \forall s \in \cS: \pi_w(s)=\pi^\dagger(s). \label{eqn:feasibility}
\end{eqnarray}
The optimal solution $|T^*|$ of this optimization problem~\eqref{eq:cardT} is called the Teaching Dimension(TD) of the problem instance $(\cS, \cA,\phi,\pi^\dagger)$ and the corresponding set  in equation~\eqref{eq:Dstar} is the minimum teaching set $D^*=D_{T^*}$.
In other words, the teacher is required to find the smallest set of states $T$ and demonstrate the target policy $\pi^\dagger$ on $T$, so that the learner learns $\pi^\dagger$ on all states.

\begin{figure}
	\centering
	\includegraphics[width=0.35\textwidth]{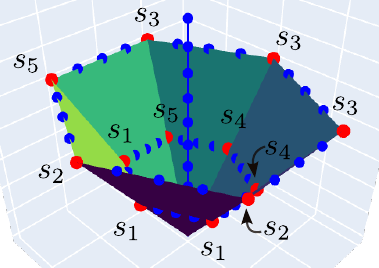}
	\caption{Optimal teaching can be hard as soon as $d>2$  where there can be an arbitrary number of extreme rays. Minimally covering all these extreme rays using a subset of states requires solving a set cover problem. In this figure each red dot $s_i$ represents a difference vector $\psi_{s_i,\pi^\dagger(s_i),b}$ for different $b$'s.
$\{s_1 \ldots s_6\}$, $\{s_2, s_3, s_5\}$ and $\{s_1,s_3\}$ are all successful teaching sets because they cover the extreme rays, but $\{s_1,s_3\}$ is the only optimal one. }
	\label{fig:optimal_teaching_3d}
\end{figure}

In Example~\ref{exp:2d_example}, we observe that $D$ is the full demonstration set under $\pi^\dagger$ i.e. $D=D_\cS$. The subset $D^\flat$ is not a valid/feasible teaching set, as its version space $\VS(D^\flat)$(shown in green in Figure~\ref{fig:non_teaching}) is wider than $\VS(D)$ and contains some $w$'s that do not induce $\pi^\dagger$ in all states, thus violating the feasibility condition in equation \ref{eqn:feasibility}. On the other hand both $D'$ and $D''$ induces the correct  version space $\VS(D_\cS)$ (as shown in green in Figures~\ref{fig:suboptimal_teaching}~and~\ref{fig:optimal_teaching}) on the learner and succeeds in teaching  $\pi^\dagger$ to it. Furthermore, $D''$ is the smallest set among them to do so, and hence it forms the optimal teaching set.

From this simple 2D example, finding the optimal set may look easy for the teacher. However, the problem quickly becomes challenging as we move to 3D or higher dimensions where there can be an arbitrary number of faces in the version space $\VS(D_\cS)$, and correspondingly $\cP(\Psi(D_\cS))$ with an arbitrary number of extreme rays, see Figure~\ref{fig:optimal_teaching_3d}. 
This difficulty is real in the sense that it leads to an NP-hard problem,
as we show in the next section.

%% file: teaching_algorithm.tex
\section{Algorithm and Analysis} 

\subsection{Teaching as an Infinite Set Cover Problem in \texorpdfstring{$w$}{w}}
Demonstrating $(s,\pi^\dagger(s))$ at a state $s$ implies $A-1$ inequalities: $\langle w, \phi(s,\pi^\dagger(s)) \rangle > \langle w, \phi(s,b) \rangle, \forall b\neq \pi^\dagger(s)$.
Each such inequality, rewritten as $\langle w, \psi_{s\pi^\dagger(s)b} \rangle > 0$, eliminates a halfspace $W_{sb}:= \{w : \langle w, \psi_{s\pi^\dagger(s)b} \rangle \leq 0\}$.
Therefore, the effect of demonstrating $(s,\pi^\dagger(s))$ is to eliminate the set of weights $W_s := \cup_{b \neq \pi^\dagger(s)} W_{sb}$.
The full demonstration set $D_\cS=\cup_{s\in\cS}\{(s,\pi^\dagger(s))\}$ over all states eliminates the union $\cup_{s\in\cS} W_s$, such that only the version space $\VS(D_\cS)$ survives.
We are interested in finding the smallest demonstration set that also produces $\VS(D_\cS)$.

We may study the problem from the opposite angle: let the universe be $\VS(D_\cS)^C$, namely the complement of the version space.
Each state demonstration $(s,\pi^\dagger(s))$ is associated with the subset $W_s$ that will be eliminated by this demonstration.
There are as many subsets as the number of states.
Then, we aim to find the smallest number of subsets whose union covers the universe.
This is a set cover problem:
\begin{eqnarray}
\min_{T\subseteq \cS}  && |T| \label{eq:infsetcover}\\
\mbox{s.t.} && \VS(D_\cS)^C = \cup_{t\in T} W_t.
\end{eqnarray}
However, both $\VS(D_\cS)^C$ and $W_t$ are infinite sets of weights.
It is not immediately clear how to solve this infinite set cover problem.
In the next section, we utilize the duality between the $w$ space (where the version space lives) and the feature space (where $\phi, \psi$ live) to convert this problem into a standard finite set cover problem. In particular, we will show that the problem is equivalent to covering the extreme rays of $\cP(\Psi(D_\cS))$.

\subsection{Teaching as a Set Cover Problem on the Extreme Rays of \texorpdfstring{$\cP(\Psi(D_\cS))$}{Psi(D)}}

We recall that the version space $\VS(D_\mathcal S)$ is induced by the feature difference set $\Psi(D_\cS)$.
In this section, we show that the set cover problem for the teacher can be simplified to covering only the extreme rays of $\cP(\Psi(D_\cS))$ using the feature difference vector induced by a subset of states $T \subseteq \cS$.

\begin{definition}[Extreme Ray and its Cover]
An open ray in $\bR^d$ is defined by the set $\cR = \{cv : c > 0\}$ for some $v\neq 0, v \in \bR^d$. 
A ray $\cR$ is called an extreme ray of a cone $K\subseteq \bR^d$ if for any $x,y\in K$, $x+y\in \cR \implies x,y\in \cR$.
We say that a state $s \in \cS$ covers a ray $\cR$ if $\exists b \neq \pi^\dagger(s) : \psi_{s\pi^\dagger(s)b} \in \cR$. Similarly, $T \subseteq \cS$ is said to cover $\cR$ if $\exists s \in T$ that covers $\cR$.
\end{definition}

We first show that it is not necessary to teach with the full demonstration set $D_\cS$ or, equivalently, the full set of induced difference vectors $\Psi(D_\cS)$.
Instead, we just need a subset $U \subseteq \Psi(D_\cS)$ that covers the extreme rays of $\cP(\Psi(D_\cS))$ as shown by Lemma~\ref{lem:extreme_ray_coverage_for_teaching}. Before we show that, we characterize the coverage condition of a cone in terms of its extreme rays in Lemma \ref{lem:extreme_ray_coverage_for_inducing_target_cone}. The proof for both the lemmas can be found in the appendix.

\begin{lemma}\label{lem:extreme_ray_coverage_for_inducing_target_cone}
	For any $U \subseteq \Psi(D_\cS), \cP(U) = \cP(\Psi(D_\cS))$ and correspondingly $\cD(U) = \cD(\Psi(D_\cS))$ if and only if $U$ contains at least one vector on each of the extreme rays of $\cP(\Psi(D_\cS))$. 
\end{lemma}

\begin{lemma}[Necessary and Sufficient condition for Teaching LBC Learner]\label{lem:extreme_ray_coverage_for_teaching}
	For successful teaching using $T \subseteq \cS$, the teacher needs to cover each extreme ray of $\cP(\Psi(D_\cS))$ using the feature difference vectors induced by teaching $\pi^\dagger$ on $T$.
\end{lemma}
\begin{proof}[Proof Sketch]
	Teaching is successful if the learner can recover a non-empty subset of target consistent version space i.e. for a teaching subset of states $T \subseteq \cS, \VS(D_T) = \VS(D_\cS)$. We also have that $\VS(D_T) \supseteq \VS(D_\cS)$ by Definition \ref{eqn:polyhedral_cone} and hence for successful teaching, the teacher has to induce complete $\VS(D_\cS)$ on the learner. From Lemma \ref{lem:extreme_ray_coverage_for_inducing_target_cone},  teaching is successful i.e. the learner recovers the complete $\VS(D_\cS) = \cD(\Psi(D_\cS))$ if and only if the teacher is able to cover(induce at least one vector on) each extreme ray of $\cP(\Psi(D_\cS))$. 
\end{proof}
This means any subset of states $T \subseteq \cS$ will be a successful demonstration if $T$ induces difference vectors that cover the extreme rays of $\cP(\Psi(D_\cS))$. For example in Figure~\ref{fig:optimal_teaching_3d}, any subset of states that can induce one red point on each of the six extreme rays is necessary and sufficient for teaching. On the other hand, the vectors on faces or in the interior of the cone (shown by blue dots in Figure~\ref{fig:optimal_teaching_3d}) are not necessary to be induced for successful teaching and hence can be discarded. The subsets $\cS, \{s_2, s_4, s_5\}, \{s_1,s_3\}$ with $\{s_1,s_3\}$ are some of the valid teaching sets with $\{s_1,s_3\}$ being the optimal one. Thus, by being clever the teacher can drastically cut the teaching cost from 30 to 2.

Once the teacher has all extreme rays, it needs to find the smallest subset of states $T \subseteq \cS$ that cover all those extreme rays.
This is our new set cover problem: The universe is the set of all extreme rays of $\cP(\Psi(D_\cS))$;
Each state $s$ is associated with the subset of extreme rays that $s$ covers.
The teacher wants to choose the smallest number of states to cover the universe.
Unlike the infinite set cover problem~\eqref{eq:infsetcover}, this is a standard finite set cover problem.

\subsection{The Optimal Teaching Algorithm and Analysis}
With ideas from the previous section, we design a teaching algorithm with two steps:
\begin{enumerate}
\item Given a teaching problem $(\cS, \cA, \phi, \pi^\dagger)$, find the extreme rays of $\cP(\Psi(D_\cS))$.  This can be solved with a sequence of linear programs.
\item Find the smallest demonstration set that covers all these extreme rays.
\end{enumerate}
We first describe how to compute the extreme rays in Lemma~\ref{lem:extreme_rays}. The proof can be found in the appendix.
Given a set of vectors $\cX \in \bR^d$ and $x \in \cX$, the following linear program determines if $x$ is the only vector on an extreme ray of $\cP(\cX)$.
\begin{align} 
	\text{LP}(x, \cX) : \qquad \min_{w} \quad & \langle w, x \rangle \label{eqn:lp}\\
	\textrm{ s.t. }  \langle w, x' \rangle \geq 1 &\,\quad \forall x' \in \cX \backslash \{x\}
\end{align}

\begin{lemma}[Extreme Ray Test]\label{lem:extreme_rays}
The objective value of LP$(x,\cX)$ is $-\infty$ if $x \notin \cP(\cX \backslash \{x\})$, and  
strictly positive otherwise.
\end{lemma}

This test allows us to define an iterative elimination procedure to find a unique representative for all the extreme rays of $\cP(\cX)$.
If LP$(x,\cX)=-\infty$, $x$ is the only vector on some extreme ray of $\cP(\cX)$ and we should keep $x$.
If LP$(x,\cX)>0$, either $x$ is not on any extreme ray, or there is some other vector $x' \in \cX \backslash \{x\}$ that is also on the same extreme ray.
In both cases, we can eliminate $x$.
When the procedure terminates, the surviving 
subset $\cX^* \subseteq \cX$ contains exactly one vector on each extreme ray of $\cP(\cX)$.
We apply this iterative elimination procedure with $\cX=\Psi(D_\cS)$ to obtain the extreme rays for $\cP(\Psi(D_\cS))$.

Next, we note that each state $s\in \cS$ induces a set of feature difference vectors $\{\psi_{s\pi^\dagger(s)b} : b \in \cA, b\neq \pi^\dagger(s)\}$ which can cover at most $|\cA|-1$ extreme rays. The goal of the teacher is to find the minimal number of states that can cover all the extreme rays.  The teacher does that by solving a standard set cover problem.

The complete algorithm for finding the optimal LBC demonstration set is given in Algorithm~\ref{alg:TIE}.

\floatname{algorithm}{Algorithm}
\begin{algorithm}[H]
	\begin{flushleft}
		\textbf{def MinimalExtreme($\cX$): }
	\end{flushleft}
	\begin{algorithmic}[1]
		\FOR {each $x_j \in \cX$} 
		\STATE Solve LP($x_j, \cX/\{x\}$) defined by \ref{eqn:lp}
		\IF{$v_j > 0$}
		\STATE $\mathcal X \leftarrow \mathcal X \backslash x_j$	\hfill \text{ $\triangleright$ eliminate $x_j$ if not necessary}
		\ENDIF
		\ENDFOR
		\STATE return $\cX$ 	\hfill \text{ $\triangleright$ extreme vectors}
	\end{algorithmic}
	\label{alg:minimal_extreme}
	
	\begin{flushleft}
		\textbf{def OptimalTeach($\cS, \cA, \pi^\dagger, \phi$): }\\
	\end{flushleft}
	\begin{algorithmic}[1]
		\STATE let $\Psi(D_\cS) = \{\psi_{s\pi^\dagger(s)b}  \in \bR^d : s\in \cS, b\in\cA, b \neq \pi^\dagger(s)\}$ 
		\STATE $\Psi^* \leftarrow$ \textbf{MinimalExtreme}$(\Psi(D_\cS) )$
		\FOR{$s \in \mathcal S$}
		\STATE $V_s \leftarrow \left\{\psi \in \Psi^*: \exists \psi_{s\pi^\dagger(s)b} \in \Psi(D_\cS), {\hat \psi_{s\pi^\dagger(s)b}}  =  {\hat \psi} \right\}$ \hfill \text{$\triangleright\, $extreme rays covered by $s$ where $\hat x = {\frac{x}{||x||}}$.}
		\ENDFOR
		\STATE $\{V_s : s \in T^* \subseteq \cS\} \leftarrow \text{SetCover}(\Psi^*, \{V_s\}|_{s \in \mathcal S})$ \\
		\hfill $\triangleright T^*$ is smallest subset that covers all extreme rays
		\STATE teach $D^* = \{(t,\pi^\dagger(t)) : t \in T^*\}$ to the agent \hfill \text{ $\triangleright\, D^*$ is the minimum demonstration set}
	\end{algorithmic} 
	\caption{Teach using Iterative Elimination (TIE) to find the optimal LBC demonstration set}
	\label{alg:TIE}
\end{algorithm}

\begin{theorem}[Teaching Dimension]
For any LBC teaching problem instance $(\cS,\cA,\phi,\pi^\dagger)$, under realizability Assumption~\ref{assum:realizability},
Algorithm~\ref{alg:TIE} correctly finds the minimum demonstration set $D^*$,
and $|D^*|$ is the teaching dimension.
\label{thm:optimal_teaching}
\end{theorem}
\begin{proof}[Proof]
	Lemma~\ref{lem:extreme_ray_coverage_for_teaching} tells us that for a valid teaching, the teacher must induce at least one feature difference vector on each of the extreme rays of $\cP(\Psi(D_\cS))$. Our teaching algorithm first finds unique representatives for each extreme ray of $\cP(\Psi(D_\cS))$ by using an iterative elimination \textbf{MinimalExtreme} procedure in Algorithm \ref{alg:TIE}. Next, it find the smallest subset of states $T^* \subseteq \cS$ that can cover all the extreme rays using a set cover problem defined in line 4-7 of \textbf{OptimalTeach} procedure in Algorithm \ref{alg:TIE}. The optimal solution of set cover $T^* \subseteq \cS$ forms the optimal solution for teaching problem as well. 
\end{proof}

TIE involves solving a set cover problem which is NP-hard to solve in general. We show that no optimal algorithm can avoid this hardness by proving that optimal teaching LBC is an NP-hard problem.

\begin{theorem}
Finding a minimum demonstration set of LBC is NP-hard.
\end{theorem}
\begin{proof}
    We provide a poly-time reduction from the set cover problem to the optimal teaching LBC problem \ref{eqn:optimal_teach}. Since the set cover is an NP-hard problem, this implies that optimal teaching is NP-hard to solve as well. 
Let $P = (U, \{V_i\}_{i\in[n]})$ be an instance of set cover problem  where $U$ is the universe and $\{V_i\}_{i\in[n]}$ is a collection of subsets of $U$. We transform $P$ into an instance of optimal teaching problem $Q = (\mathcal S, \mathcal A, \pi^\dagger, \phi)$. 

Construction: For each subset $V_i$ of $P$, we create a state $s_i$ of $Q$. For each element $k$ in the universe $U$ of $P$, we create an extreme ray vector $\psi_k$ of feature difference vectors in $Q$. The complete construction is given as follows :

\begin{enumerate}
	\item $\mathcal S = [n], \mathcal A = \left[A\right]$ where $A =  \max_{i \in [n]} |V_i| + 1$.
	\item The target policy is $\pi^\dagger(s) = A,\, \forall s \in \cS$.
	\item $\Psi = \{\psi_k = (\cos({\frac{2\pi k}{n}}), \sin(\frac{2\pi k}{n}), 10) : k \in [|U|]\}$.
	\item for each $s \in \mathcal S$ we construct feature vectors $\{\phi(s,a) : a \in \mathcal A\}$ such that the feature differences map to extreme rays $\psi$'s. Enumerating over element of $V_s := \{V_{s1},\cdots, V_{s|V_s|}\}$, we define the induced feature difference vectors as,
\begin{align}
	\psi_{sAb} &= \psi_{V_{sb}},\, \hspace*{0.35cm}\forall b < |V_s|-1 \\ 
	\psi_{sAb} &= \psi_{V_{s|V_s|}},\, \forall  |V_s|-1\leq b \leq A-1
\end{align}
\end{enumerate}

\begin{figure}[H]
	\centering
	\includegraphics[width=0.65\textwidth]{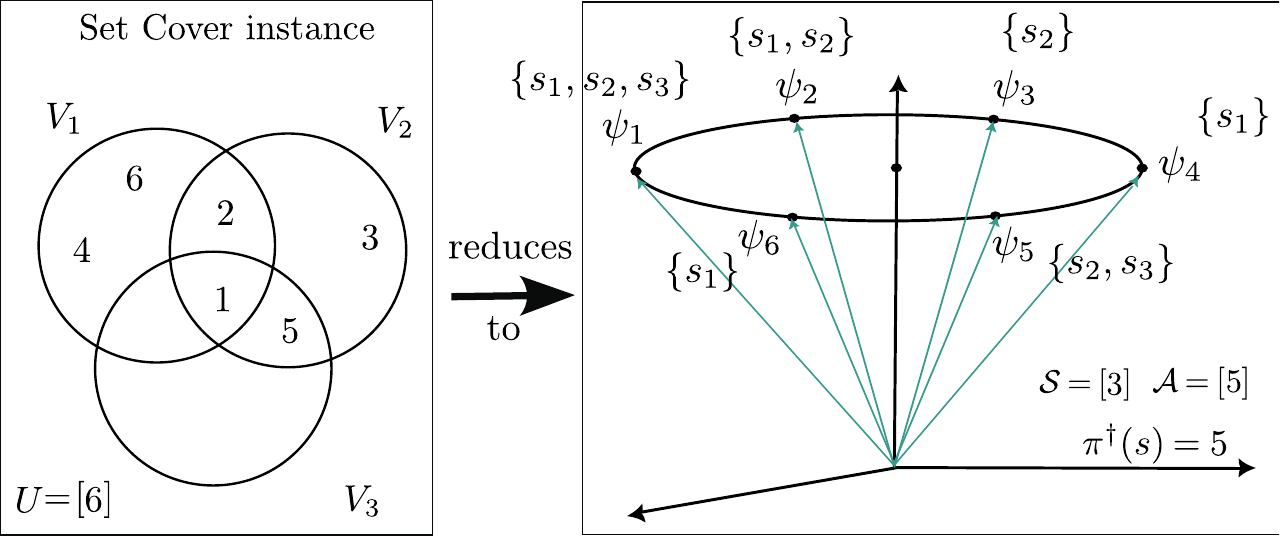}
	\caption{An reduction example from a set cover problem to optimal teaching LBC problem}
	\label{fig:teaching_hardness}
\end{figure}

\begin{claim}
	A solution of optimal teaching LBC instance  $(\mathcal S, \mathcal A, \pi^\dagger, \phi)$ gives a solution to the set cover problem $(U, \{V_i\}_{i\in[n]})$ and vice versa.
\end{claim}
Finding a collection of subsets $\{V_i\}_{i\in [n]}$ of smallest size that covers all elements in universe $U$ is equivalent to selecting a subset of states $\mathcal S$ of smallest size that covers all the extreme rays defined by $\Psi$. 
	
For a solution $\{V_j\}_{j \in T^*} \text{ s.t. } T^* \subseteq [n]$ to the set cover instance $(U, \{V_i\}_{i\in[n]})$, the set of states indexed by $T^* \subseteq \mathcal S$ is a solution to the optimal teaching instance $(\mathcal S, \mathcal A, \pi^\dagger, \phi)$ and vice versa. The argument follows from a direct translation between two instances. See Figure~\ref{fig:teaching_hardness}.
\end{proof}

With this hardness reduction, we have no hope of finding an efficient algorithm for optimally teaching LBC. However, we observe that the only ``hard'' part of TIE algorithm is to solve the set cover problem in line 6 of OptimalTeach procedure in Algorithm~\ref{alg:TIE}. The other main step of finding the extreme rays using a sequence of linear programs is efficiently solvable. This immediately suggests using an approximate algorithm to solve the set cover problem, which gives us an efficient but slightly suboptimal algorithm for optimal teaching.

We use the standard greedy algorithm for set cover which selects the next state that covers the most number of uncovered extreme rays.
It is well known that this greedy algorithm achieves an approximation ratio of $\log$(largest subset cardinality) \cite{setCover2005}.
For our problem, the cardinality of any subset is upper bounded by the number of extreme rays a single state can cover, which in turn is upper bounded by $|\cA|-1$ which is the maximum number of difference vectors that the state can produce. 
This gives us the following guarantee for efficiently solving optimal teaching problem.

\label{thm:approximate_optimal_teaching}
\begin{corollary}[Approximating Optimal Teaching]
Algorithm~\ref{alg:TIE} using the greedy set cover algorithm efficiently finds an approximate optimal teaching set $\tilde D$ such that
$|\tilde D| \leq \log(|\cA|-1) |D^*|$.
\end{corollary}

%% file: experiments.tex
\section{Experiments}

We present two LBC teaching problems to demonstrate the effectiveness of the TIE algorithm. 

\subsection{Pick the Right Diamond}
\label{exp:teach_to_pick_the_diamond} 

Recall the game from Example~\ref{exp:pick_the_diamond_in_intro}.
The state space is $\mathcal S = \{\hexagon,\pentagon,\square, \triangle,o\}^n / \{o\}^n$ where $o$ represents an empty slot, and the action space is $\mathcal A = [n]$. The learner uses a 2D feature function given as follows, 
\begin{align}
	\phi(s,a) = [a, \# \textrm{edges of diamond at } a]. \label{eqn:feature_fn_pick_the_diamond}
\end{align}
where $[\# \textrm{edges of diamond at } a]$ is 0 if there is not diamond at that slot.

\begin{figure}[H]
	\begin{subfigure}{0.5\textwidth}
		\includegraphics[width=0.9\textwidth]{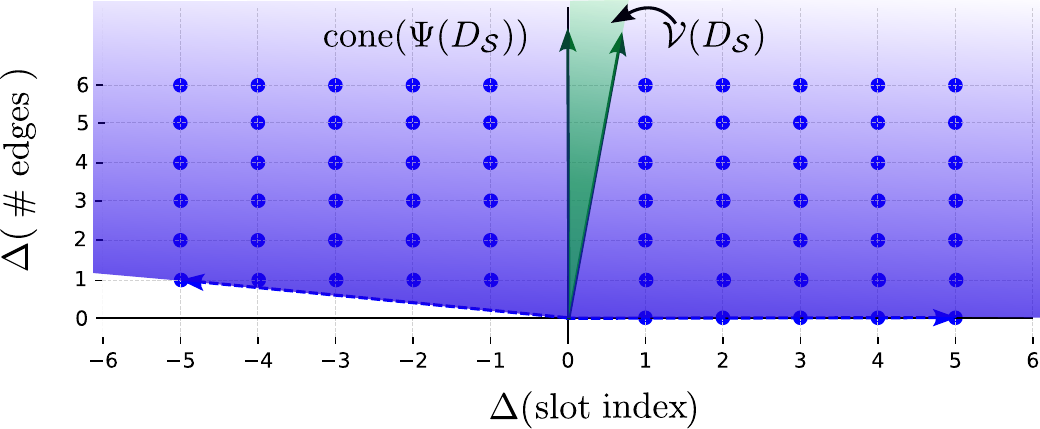}
		\caption{}
		\label{fig:example_pick_the_diamond_cone}
	\end{subfigure}
	\begin{subfigure}{0.5\textwidth}
		\centering
		\includegraphics[width=0.9\textwidth]{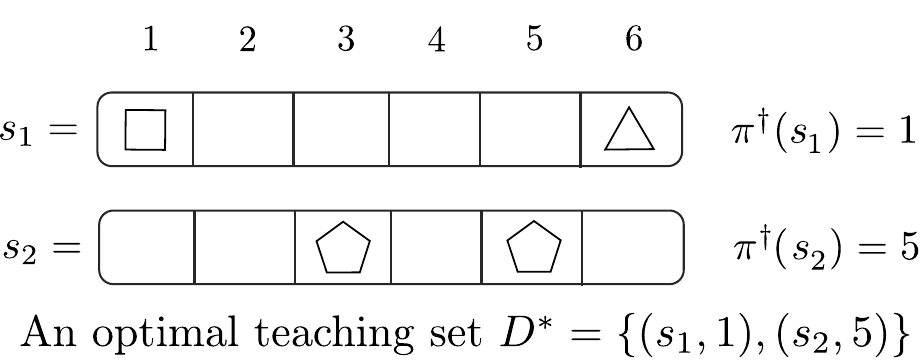}
		\caption{}
		\label{fig:teaching_set_on_pick_the_diamond}
	\end{subfigure}	
	\caption{Optimal teaching in ``pick the right diamond'' with $n=6$ slots.
	(a) Feature difference vectors $\Psi(D_\cS)$ as blue dots,
	primal cone $\cP(\Psi(D_\cS))$ as blue area, and dual version space $\VS(D_\cS)$ as green area.
	(b) One minimum optimal teaching set with two states.}
\end{figure}

\begin{figure}
	\centering
	\begin{subfigure}{0.4\textwidth}
		\centering
		\includegraphics[width=\textwidth]{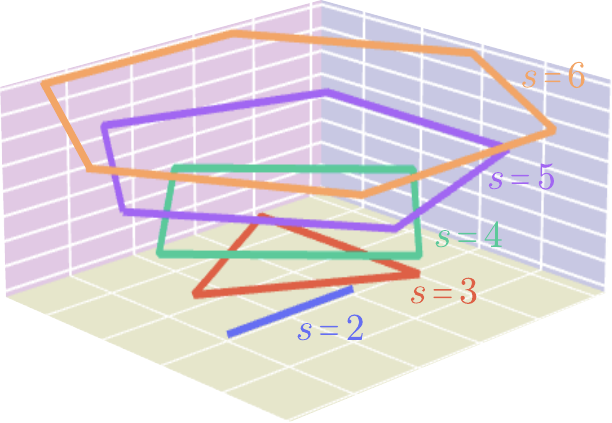}
		\caption{}
	\end{subfigure}
	\begin{subfigure}{0.35\textwidth}
		\centering
		\includegraphics[width=\textwidth]{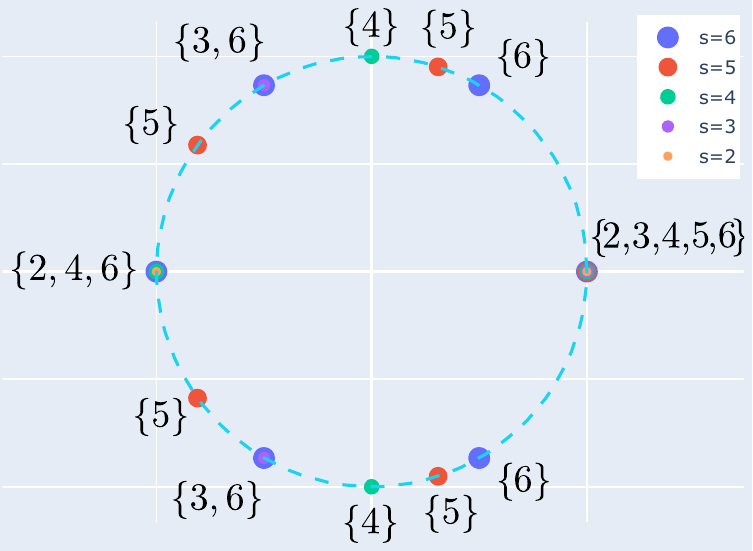}
		\caption{}
	\end{subfigure}
	
	\begin{subfigure}{0.35\textwidth}
		\centering
		\includegraphics[width=\textwidth]{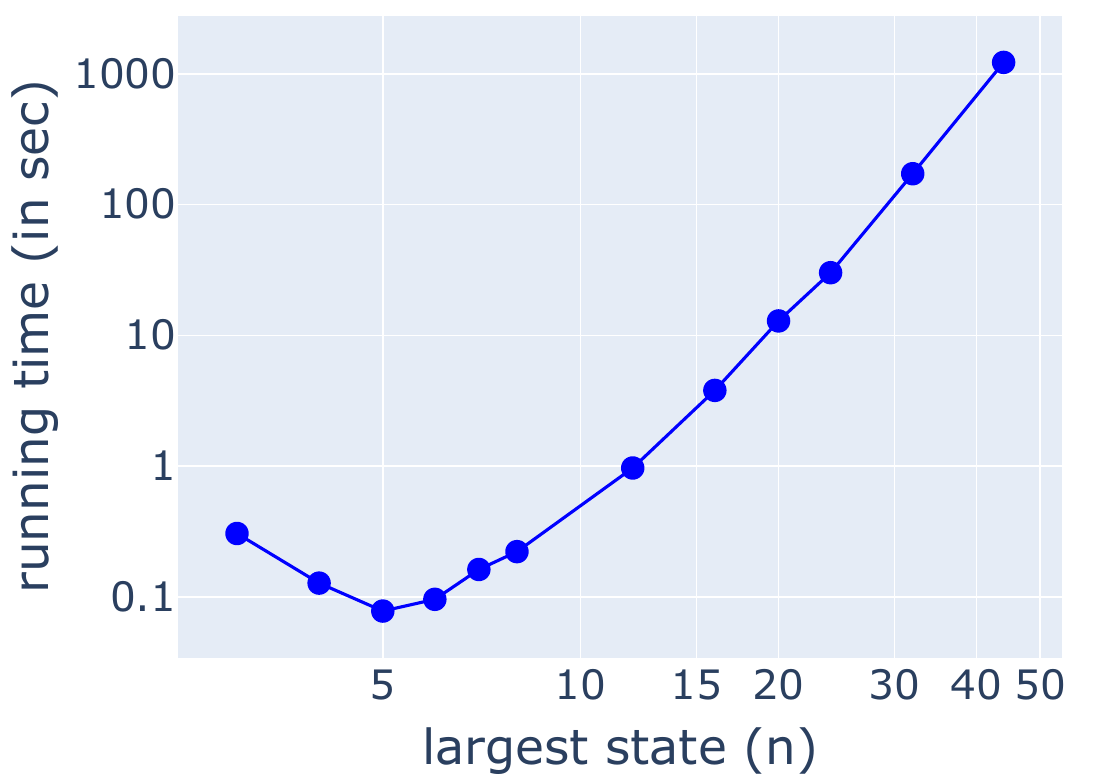}
		\caption{}
	\end{subfigure}
	\begin{subfigure}{0.35\textwidth}
		\centering
		\includegraphics[width=\textwidth]{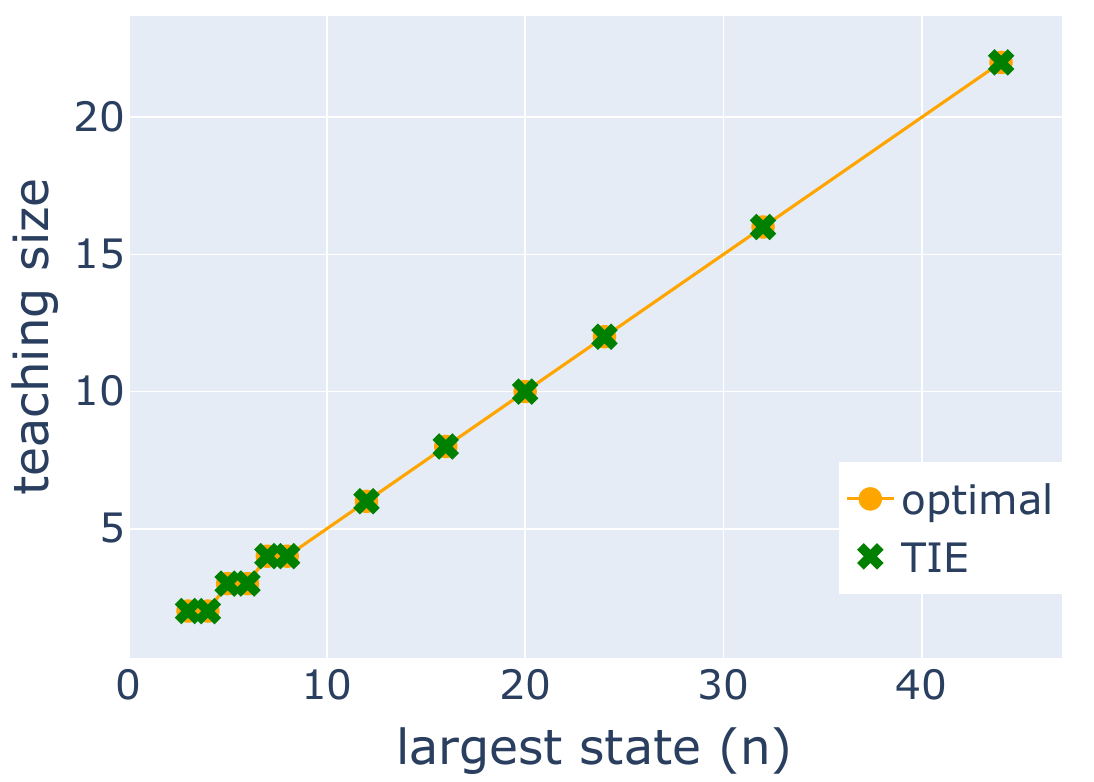}
		\caption{}
	\end{subfigure}
	
	\caption{Polygon Tower.  (a) All feature difference vectors for $n=6$.
 (b) Top-down view of the extreme vectors of the primal cone for $n=6$.
(c) TIE running time on polygon tower with increasing $n$.
(d) The teaching dimension (optimal) vs. the demonstration set size found by TIE.  They overlap.  In fact, TIE finds the exact correct optimal teaching sets on polygon tower.}
	\label{fig:example_polygon_tower}
\end{figure}

\paragraph*{Target policy}  The rule stated in Example~\ref{exp:pick_the_diamond_in_intro}
 is the target policy $\pi^\dagger$.
The target policy is realizable under $\phi$.
For example, $w^\dagger = [1,10]$ uniquely induces $\pi^\dagger$.  
There are a total of $5^6-1$ states and their feature difference vectors $\Psi(D_\cS)$ induced by the target policy are shown as blue dots in Figure~\ref{fig:example_pick_the_diamond_cone}.
The primal cone $\cP(\Psi(D_\cS))$ is the blue-shaded area.
It contains two extreme rays, both need to be covered for successful teaching. 
The version space is green.

\paragraph*{Optimal Teaching} 
Multiple demonstration sets can cover extreme rays. The left extreme ray is covered by any state in which the target action is 1 (i.e. the diamond in slot 1 is the only one having the highest number of edges), and the diamond in slot 6 has one less edge. 
The right extreme ray is covered by any state in which the are two or more diamonds with the highest number of edges and the target action is to choose the rightmost of them.
Figure~\ref{fig:teaching_set_on_pick_the_diamond} gives an example.
Demonstrating two such states induces the correct version space $\VS(D_\cS)$ in the learner. This is a dramatic improvement over demonstrating on all $5^n-1$ states.

We run TIE on this instance which produces a teaching set of size two as illustrated in Figure~\ref{fig:teaching_set_by_tie_on_pick_the_diamond}.
\begin{figure}[H]
	\begin{subfigure}{0.5\textwidth}
		\includegraphics[width=0.95\textwidth]{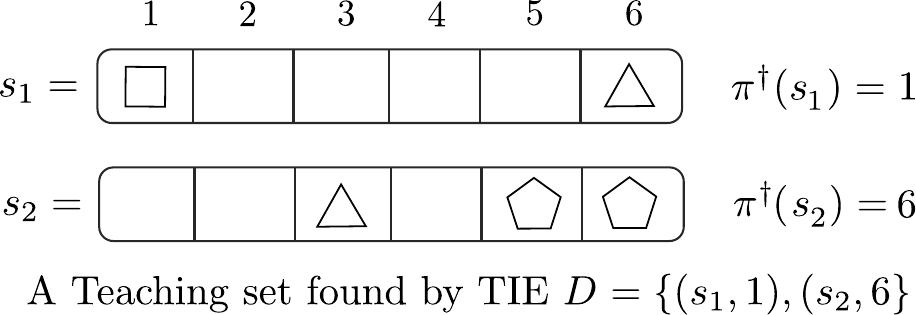}
		\caption{The teaching set produced by TIE algorithm.}
		\label{fig:teaching_set_by_tie_on_pick_the_diamond}
	\end{subfigure}
	\begin{subfigure}{0.45\textwidth}
		\centering
		\includegraphics[width=0.95\textwidth]{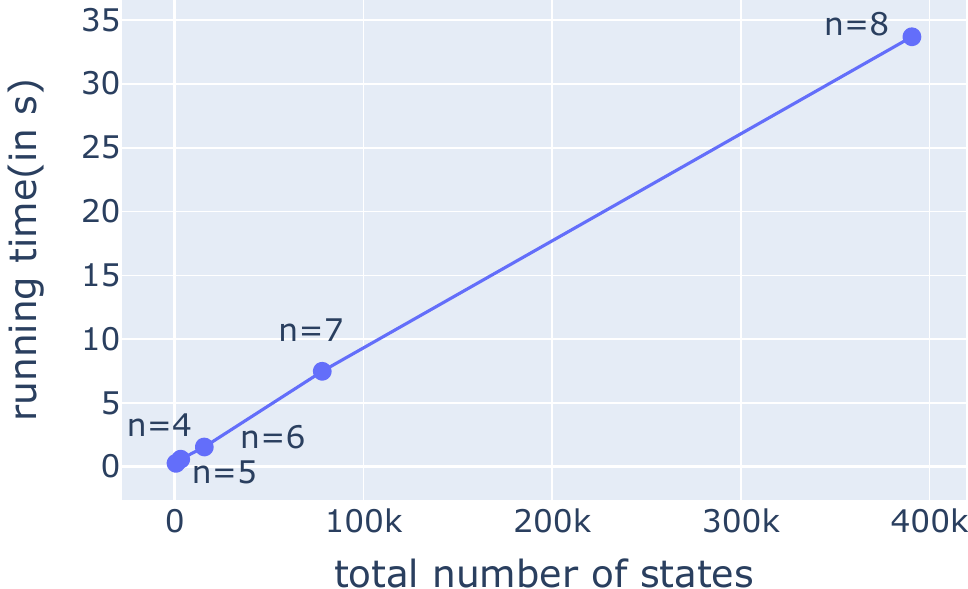}
		\caption{Runtime of TIE on pick the right diamond.}
		\label{fig:pick_the_diamond_running_time}
	\end{subfigure}	
\end{figure}
We note that that this is not only a valid teaching set but is also an optimal one. Furthermore, our algorithm scales almost linearly in the size of the state space $5^n$ as shown by the curve in Figure~\ref{fig:pick_the_diamond_running_time}. This is due to the fact that even though the state space  of this problem is huge, the number of feature difference vectors and corresponding extreme rays in this problem is small. Therefore, the linear program and the set cover problem take negligible time compared to constructing feature difference vectors in each state which is linear in $n$. This validates the effectiveness and efficiency of our algorithm on this family of instances.

\subsection{Polygon Tower}
Let the state space be $\mathcal S = \{2,\ldots,n\}$, the action space be $\mathcal A = [n+1]$,
the feature function be $\phi : \mathcal S \times \mathcal A \rightarrow \mathbb R^3$ given by
\begin{align}
	\phi(s,a) =  
	\begin{cases}
		\left[0,0,s\right] & \text{ if } a = n+1\\
		\left[-s \cdot \cos({\frac{2\pi a}{s}}), -s \cdot \sin({\frac{2\pi a}{s}}), 0\right] & \text{ otherwise}
	\end{cases}
\end{align}
We note that for a fixed state $s$, the feature vectors for actions $1\ldots n$ lie on a polygon of radius $s$ centered around the origin on the xy plane.
\paragraph*{Target Policy} The teacher wants to teach the target policy $\pi^\dagger$ where $\forall s \in \cS, \,\pi^\dagger(s) = n+1$. The policy is realizable: for example, $w = [0,0,1]$ induces this policy. 
The feature difference vectors induced by $\pi^\dagger$ on $\cS$ is given as $\Psi(D_\cS) = \{[s \cdot \cos({\frac{2\pi a}{s}}), s \cdot \sin({\frac{2\pi a}{s}}), s] : s \in \mathcal S, a \neq n+1\}$.
These difference vectors lie on elevated polygons as shown in Figure~\ref{fig:example_polygon_tower}(a).
In particular, state $s$ induces a $s$-gon of radius $s$ centered at $(0,0,s)$. Figure~\ref{fig:example_polygon_tower}(b) shows the top view of the extreme rays of the primal cone $\cP(\Psi(D_\cS))$. The extreme rays are shown as dots and the states that cover each extreme ray are labeled. 

\paragraph*{Optimal Teaching} 
The polygon tower problem has an interesting structure that allows us to characterize the minimum demonstration set.
\begin{proposition}
\label{prop:factor}
	The optimal teaching set $T^*$ of the polygon tower consists of all states in $\cS$ that are not divisible by any other states in $\cS$.
\end{proposition}
\begin{proof}
	For any pair of states $s, s'$ such that $s' > s$ and $s' \text{ mod } s =0$, then $s'$ fully covers the induced difference vectors of the characteristic of $s$ so teaching state $s$ is not required if $s'$ is taught. For example, state 6 in Figure~\ref{fig:example_polygon_tower}(b) covers all the extreme rays induced by states 2, and 3. Conversely, if a state $s$ is not a factor of any other states in $\cS$ then it must be taught because $s$ induces the extreme ray $[s \cdot \cos({\frac{2\pi}{s}}), s \cdot \sin({\frac{2\pi}{s}}), s] $ that can only be covered by $s$. 
\end{proof}

We run TIE with greedy set cover on a family of polygon towers with $n \in  \{3,4,5,6,7,8,12,16,20,24,32,44\}$. 
We verify TIE's solution to the ground truth minimum demonstration set established in Proposition~\ref{prop:factor}.

We observe that TIE always recovers the correct minimum demonstration set. This can be observed from the overlap curve of the optimal size of the teaching set (shown in orange) and the size of the teaching set found by TIE (shown in green) in Figure~\ref{fig:example_polygon_tower}(d). 

We also observe that TIE runs quickly.  We plot the running time of TIE over instance size $n$
in a log-log plot in Figure~\ref{fig:example_polygon_tower}(c). For each $n$, we average the running time over $3$ independent trial runs. The straight line of this log-log plot shows that our algorithm indeed runs in polynomial time. The empirical estimate of the slope of the linear curve (after omitting the first three outlier points for small $n$) turns out to be 4.67 implying a running time of order $O(n^5)$ on this family of instances. Our algorithm has a worst-case running time of order $O({(|\cS||\cA|)}^3)$ and for $|\cS|=|\cA|=n$ as in this example, it is $O(n^6)$.

%% file: related_works.tex
\section{Related Work}

On of the earliest work in Machine Teaching was done by \cite{goldman1992complexity} who studied Machine Teaching of a Behavior Cloning agent with a finite hypotheses space. We extend this work to a more challenging setting teaching a learner with an infinite uncountable hypotheses family. Several follow up works since then have studied machine teaching in different machine learning paradigms. For example \cite{liu2016teaching, zhu2015machine, zhu2013machine,NEURIPS2021_b5488aef,qian2022teaching} studied teaching algorithms for supervised learning agents that work under noisy demonstrations, \cite{cakmak2011active,wang2021teaching} studied optimal teaching of active learning agents and more recent works like \cite{brown2019machine,rakhsha2020policy,cakmak2012algorithmic,mansouri2019preference,rakhsha2021policy, zhang2021sample} have studied machine teaching more complex online sequential decision making learners like Q-Learning. Our work differs from these prior works in that we study machine teaching of a simple linear Behavior Cloning learner that has not been studied in the past.

Some recent works have also studied machine teaching of various kinds of linear learners. For example, \cite{ohannessianoptimal} studied optimal teaching of linear perceptron algorithm,  \cite{liu2016teaching,Ma2018TeacherIL} studied teaching linear empirical risk minimizing algorithms like Ridge Regression etc. These works address various types of linear learners, all of which are clearly very different the Linear Behavior Cloning learner that we study in this paper. Another line of works ~\cite{cakmak2012algorithmic, walsh2012dynamic, DBLP:conf/nips/TschiatschekGHD19,DBLP:conf/ijcai/KamalarubanDCS19,brown2019machine,haug2018teaching}  have studied Teaching by Demonstration in Reinforcement Learning setting. However, unlike our LBC learner who learns a policy directly from the demonstration, they studied learners like Inverse Reinforcement Learning which first learns a reward function $R$ from the demonstrations and then finds an optimal policy with respect $R$. These learners are clearly different from ours.

%% file: conclusion.tex
\section{Conclusion}
In this work, we studied the problem of optimal teaching of a Linear Behavior Cloning (LBC) learner. We characterized the necessary and sufficient condition for teaching in terms of covering extreme rays and provided an optimal algorithm called TIE to find the optimal teaching set. We further proved that optimal teaching LBC is a NP-hard problem and provided an efficient but approximate algorithm that finds an approximately optimal teaching set with an approximation ratio of $\log(|\cA|-1)$. Finally, we demonstrated the effectiveness of our algorithm on two different instances.

%% file: appendix.tex
\section{Appendix}

Given a finite set of vectors $\cX = \{x_i  \in \bR^d: i \in [n]\}$,  we define the primal cone generated by this set as 
\begin{equation} \label{eq:conedef}
  \cP(\cX) = \left\{  \sum_{i \in \cX} \lambda_i x_i, \;\; \lambda_i \ge 0, \; \mbox{ $\forall i \in \cX$} \right\}.
\end{equation}
Given any set $U$, we define the dual cone as
\begin{equation} \label{eq:ClU}
\cD(U) := \{ y \, | \, y^Tx >0,\; \mbox{$\forall x \in U, x \ne 0$} \}.
\end{equation}
In particular, if the finite set $\cX$ has $x_i \ne 0$ for all $i \in [n]$, we have
\begin{equation} \label{eq:C1}
\cD(\cX) := \{ y \, | \, y^Tx_i >0, \, i=1,2,\dotsc,n \}.
\end{equation}

We prove some basic properties about cones in $\bR^d$. 
\begin{proposition}For any finite sets $U,V \text{ s.t. }U \subseteq V \subset \bR^d$, we have that,
	\begin{enumerate}
		\item  $\cD(U) = \cD(\cP(U))$.\\ \label{eqn:prop1}
		\item $\cD(U) \supseteq \cD(V)$.\\ \label{eqn:prop2}
		\item $\cP(U) = \cP(V) \implies \cD(U) = \cD(V)$.\\ \label{eqn:prop3}
	\end{enumerate}
\end{proposition}
\begin{proof}
	\ref{eqn:prop1} For any $w \in \cD(U), \langle w, u_i\rangle > 0,\,\forall u_i \in U \implies \forall i, \lambda_i \geq 0, \sum_i \lambda_i u_i \neq 0,\,  \langle w, \sum_i \lambda_i u_i\rangle > 0 \implies w \in \cD(\cP(U))$. For the opposite direction, let  $\forall \lambda_i \geq 0, \sum_i \lambda_i u_i \neq 0,\, \langle w, \sum_i \lambda_i u_i\rangle > 0$. For a fixed $i$, choose $\lambda_i=1$ and $\lambda_j =0,\,\forall j \neq i$. Then, we have $\langle w, u_i\rangle > 0,\forall u_i \in U$, thus, $w \in \cD(U)$.
	
	\ref{eqn:prop2} Now, for second part of the proposition, let $x \in \cD(V)$ i.e. $\langle x, v \rangle > 0,\,\forall v \in V$.  Since, $U \subseteq V$, this implies $\langle w, u_i\rangle > 0,\forall u_i \in U$. Thus, $x \in \cD(U)$, thus proving the statement. 
	
	\ref{eqn:prop3} Finally, for the the third part, we have that $\cD(U)=\cD(\cP(U))=\cD(\cP(V))=\cD(V)$, where the first and third equality follows from part 1 of this proposition and second equality follows from the premise.	
\end{proof}

For proof of Lemma \ref{lem:extreme_ray_coverage_for_inducing_target_cone} in the main text, apply Lemma~\ref{lem:lemma_1_proof} with $V = \Psi(D_\cS)$.

\newcommand{\R}{\mathbb{R}}

\newtheorem{thm}{Theorem}

\newtheorem{prop}{Proposition}
\newtheorem{cor}{Corollary}
\newtheorem{ex}{Example}
\newtheorem{exse}{Exercise}
\newtheorem{defn}{Definition}

\subsection{Finding extreme rays of primal cone}

In the remainder, we assume that the finite set $\cX = \{x_i  \in \bR^d: i \in [n]\}$ contains all nonzero vectors, and recall the definitions of cone \eqref{eq:conedef} and dual cone \eqref{eq:C1}.
Our problem is to find a set $\cX^* \subset \cX$ of minimum cardinality such that $\cD(\cX^*) = \cD(\cX)$.

Note that by realizability, we have that $\cD(\cX)$ is nonempty. We can define $\cD(\cX)$ alternatively as follows
\begin{equation} \label{eq:Cdef}
	\begin{aligned}
		\cD(\cX) & := \{  \alpha z \, | \, \alpha>0, \,  z\in P(\cX) \} \\
		\mbox{where} \;\; P(\cX) & := \{ z \, | \, z^Tx_i \ge 1, \, i \in \cX \}.
	\end{aligned}
\end{equation}
\begin{proof}
	Any $z$ satisfying \eqref{eq:Cdef} clearly has $z^T x_i >0$
	for all $i \in \cX$, so $z \in \cD(\cX)$. Conversely, given any $y$ with $y^Tx_i>0$ for
	all $i \in \cX$, we set $\alpha = \min_{i \in \cX} \,
	y^Tx_i>0$ and $z = y/\alpha$ to get $\alpha$ and $z$ satisfying
	\eqref{eq:Cdef}.
\end{proof}

The key element of the algorithm is an LP of the following form, for
some $x_j \in \cX$:
\begin{equation} \label{eq:jv1}
	\begin{aligned}
		 \mbox{LP}(x_j,\cX/\{x_j\}): \qquad \qquad  \min_{w} \, w^T x_j \;\\ 
		 \mbox{subject to} \; w^Tx_i \ge 1 \; \mbox{$\forall i \in \cX/\{x_j\}$.}
	\end{aligned}
\end{equation}
Note that this problem can be written alternatively, using the notation of \eqref{eq:Cdef}, as
\begin{equation} \label{eq:jv1a}
 \min_{w} \, w^T x_j \; \mbox{subject to }  w \in P(\cX/\{x_j\}).
\end{equation}
The dual of \eqref{eq:jv1} will also be useful in motivating and
understanding the approach:
\begin{equation} \label{eq:jv.dual}
	\begin{aligned}
		\mbox{LP-Dual}(x_j,\cX/\{x_j\}): \quad
		\max_{\lambda_i, i \in \cX/\{x_j\}} \, \sum_{i \in \cX/\{x_j\}} \lambda_i \;\;\\
		\mbox{s.t.} \;\;  \sum_{i \in \cX/\{x_j\}} \lambda_i x_i = x_j, \;\; \lambda_i \ge 0 \; \mbox{for all $i \in \cX/\{x_j\}$.}
	\end{aligned}
  \end{equation}

We prove a lemma with several observations. 
\begin{lemma}[Proof of Lemma 3] \label{lem:dual}
 Suppose that $\cX$ is not empty. Let $x_j \in \cX$. We have the following.
  \begin{enumerate}
    \item[(i)] When \eqref{eq:jv1} is unbounded, \eqref{eq:jv.dual} is
      infeasible, so $x_j \notin \cP (\cX/\{x_j\})$. Furthermore, $\exists w \in \bR^d$ s.t. $w \in \cD(\cX/\{x_j\})$ but $w \notin \cD(\cX)$.
    \item[(ii)] if \eqref{eq:jv1} has a solution, the optimal objective value must be positive.
    \item[(iii)] When \eqref{eq:jv1} has a solution with a positive optimal objective, then $x_j \in \cP (\cX/\{x_j\})$.
  \end{enumerate}
\end{lemma}
\begin{proof}
  \begin{itemize}
  \item[(i)] From LP duality, when \eqref{eq:jv1} is unbounded, then
  \eqref{eq:jv.dual} is infeasible, giving the first part of the result. 
  For the second part, we note by the feasibility condition of \ref{eq:jv1} that the optimal solution $w^* \in \cD(\cX/\{x_j\})$ but since ${w^*}^Tx_j  <0$, that is, $w^* \notin \cD(\cX)$, giving us the result.

  \item[(ii)] If \eqref{eq:jv1} were to have a solution with optimal
  objective $0$, then by LP duality, the optimal objective of
  \eqref{eq:jv.dual} would also be zero, so the only possible value
  for $\lambda$ is $\lambda_i=0$ for all $i \in \cX/\{x_j\}$. The constraint of \eqref{eq:jv.dual} then implies that $x_j=0$, which cannot be the case, since we assume that all vectors in $\cX$ are nonzero.
  
  \eqref{eq:jv1} cannot have a solution with {\em negative} optimal objective value, because by LP duality, \eqref{eq:jv.dual} would
  also have a solution with negative objective value. However, the value of the objective for \eqref{eq:jv.dual} is non-negative at all feasible points, so this cannot happen.

  \item[(iii)] When \eqref{eq:jv1} has a solution with positive optimal
  objective, then LP duality implies that \eqref{eq:jv.dual} has a
  solution with the same objective. Thus, there are nonnegative $\lambda_i$, $i \in \cX/\{x_j\}$, not all $0$, such that the constraint in \eqref{eq:jv.dual} is satisfied, giving the result. \label{eqn:duality_result}
  \end{itemize}
\end{proof}

\begin{lemma}[Proof of Lemma 1]
	Let $U$ and $V$ be finite sets with $U \subseteq V \subseteq \bR^d$ and $\cD(V)$ is non-empty. Then $\cP(U) = \cP(V)$ and $\cD(U) = \cD(V)$ if and only if $U$ contains at least one vector on each of the extreme rays of $\cP(V)$. 
	\label{lem:lemma_1_proof}
\end{lemma}

\begin{proof}
	For the sufficiency direction, we note that for a set $U \subseteq V$, if $U$ contains at least one vector on each of the extreme rays of $\cP(V)$ then $\cP(U) = \cP(V)$ (since all the vectors in a cone can be expressed as a conic combination of extreme vectors of the cone). Furthermore, by Proposition \ref{eqn:prop3}, we have $\cD(U) = \cD(V)$. 
	
	For necessity direction, suppose that $U$ does not contain any vector on a certain extreme ray $\cR$ of $V$. 
 Then $U \subseteq V/\cR$. Thus $\cP(U) \subseteq \cP(V/\cR) \subsetneq \cP(V)$ and thus $\cD(U) \supseteq \cD(V/\cR)$. 
 Let $r \in \cR \subset V$. Then LP-Dual$(r,U)$ will be infeasible, so LP$(r,U)$ is either infeasible or unbounded. But LP$(r,U)$ is feasible because pointedness of $\cP(V)$ means that $\cD(V)$ is nonempty, so $\cD(U) \supset \cD(V)$ is also nonempty and so the constraints of LP$(r,U)$ are guaranteed to hold for some $w$. We conclude that LP$(r,U)$ is unbounded, so there is a direction $w$ such that $w^Tx > 0$ for all $x \in U$ but $w^T r<0$. This vector $w$ belongs to $\cD(U)$ but not to $\cD(V)$, so $\cD(U) \subsetneq \cD(V)$, as required.
\end{proof}

\begin{theorem}[Proof of Theorem 4]
	For any LBC teaching problem instance $(\cS,\cA,\phi,\pi^\dagger)$, under realizability Assumption~\ref{assum:realizability},
	Algorithm~\ref{alg:TIE} correctly finds the minimum demonstration set $D^*$,
	and $|D^*|$ is the teaching dimension.
	\label{thm:optimal_teaching_complete}
\end{theorem}
\begin{proof}
	Lemma~\ref{lem:extreme_ray_coverage_for_teaching} tells us that for a valid teaching, the teacher must induce at least one feature difference vector on each of the extreme rays of $\cP(\Psi(D_\cS))$. The iterative elimination procedure in \textbf{MinimalExtreme} in Algorithm \ref{alg:TIE} first finds unique representatives for each extreme ray of $\cP(\Psi(D_\cS))$. This follows from lemma~\ref{lem:dual}. Let $\cX$ be the surviving set of vectors at the start of an iteration where $x_j$ is considered. We have that if $x_j \in \cP(\cX/\{x\})$ it will get eliminated by the extreme ray test \ref{lem:extreme_rays} and on the other hand if $x_j$ is unique representative for an extreme ray in $\cX$, we have $x_j \notin \cP(\cX/\{x\})$ and thus $x_j$ will not get eliminated. At every iteration, we either eliminate a vector in $\Psi(D_\cS)$ or that vector is a unique representative for an extreme ray of $\cP(\Psi(D_\cS))$ and cannot be eliminated. Thus, at the end of the iterative elimination procedure, we recover a set of unique representative vectors $\Psi^*$ for each extreme ray of $\cP(\Psi(D_\cS))$.
	
	Next step involves finding a smallest subset of states $T \subseteq \cS$ that can cover all the extreme rays. This is done by a set cover problem defined on lines 4-7 of the \textbf{OptimalTeach} procedure in Algorithm \ref{alg:TIE}. The set of unique representatives of extreme rays forms the universe to be covered and each state defines a subset of representatives for extreme rays that it can cover. The minimum number of subsets that can cover the entire universe is the minimum number of states that covers all the extreme rays giving us $T^* \subseteq \cS$ as an optimal solution for teaching problem. 
\end{proof}

%% file: main_arxiv.bbl
\begin{thebibliography}{10}

\bibitem{brown2019machine}
Daniel~S Brown and Scott Niekum.
\newblock Machine teaching for inverse reinforcement learning: Algorithms and
  applications.
\newblock In {\em Proceedings of the AAAI Conference on Artificial
  Intelligence}, volume~33, pages 7749--7758, 2019.

\bibitem{cakmak2012algorithmic}
Maya Cakmak and Manuel Lopes.
\newblock Algorithmic and human teaching of sequential decision tasks.
\newblock In {\em Twenty-Sixth AAAI Conference on Artificial Intelligence},
  2012.

\bibitem{cakmak2011active}
Maya Cakmak and Andrea~L Thomaz.
\newblock Active learning with mixed query types in learning from
  demonstration.
\newblock In {\em Proc. of the ICML workshop on new developments in imitation
  learning}. Citeseer, 2011.

\bibitem{Codevilla_2019_ICCV}
Felipe Codevilla, Eder Santana, Antonio~M. Lopez, and Adrien Gaidon.
\newblock Exploring the limitations of behavior cloning for autonomous driving.
\newblock In {\em Proceedings of the IEEE/CVF International Conference on
  Computer Vision (ICCV)}, October 2019.

\bibitem{goldman1992complexity}
Sally~A Goldman and Michael~J Kearns.
\newblock On the complexity of teaching.
\newblock {\em Journal of Computer and System Sciences}, 50(1):20--31, 1995.

\bibitem{goldman1993teaching}
Sally~A Goldman and H~David Mathias.
\newblock Teaching a smart learner.
\newblock In {\em Proceedings of the sixth annual conference on computational
  learning theory}, pages 67--76, 1993.

\bibitem{haug2018teaching}
Luis Haug, Sebastian Tschiatschek, and Adish Singla.
\newblock Teaching inverse reinforcement learners via features and
  demonstrations.
\newblock In {\em Advances in Neural Information Processing Systems}, pages
  8464--8473, 2018.

\bibitem{DBLP:conf/ijcai/KamalarubanDCS19}
Parameswaran Kamalaruban, Rati Devidze, Volkan Cevher, and Adish Singla.
\newblock Interactive teaching algorithms for inverse reinforcement learning.
\newblock In {\em IJCAI}, pages 2692--2700, 2019.

\bibitem{liu2016teaching}
Ji~Liu, Xiaojin Zhu, and Hrag Ohannessian.
\newblock The teaching dimension of linear learners.
\newblock In {\em International Conference on Machine Learning}, pages
  117--126. PMLR, 2016.

\bibitem{NEURIPS2021_b5488aef}
Weiyang Liu, Zhen Liu, Hanchen Wang, Liam Paull, Bernhard Sch\"{o}lkopf, and
  Adrian Weller.
\newblock Iterative teaching by label synthesis.
\newblock In M.~Ranzato, A.~Beygelzimer, Y.~Dauphin, P.S. Liang, and J.~Wortman
  Vaughan, editors, {\em Advances in Neural Information Processing Systems},
  volume~34, pages 21681--21695. Curran Associates, Inc., 2021.

\bibitem{Ma2018TeacherIL}
Yuzhe Ma, Robert~D. Nowak, Philippe Rigollet, Xuezhou Zhang, and Xiaojin Zhu.
\newblock Teacher improves learning by selecting a training subset.
\newblock In {\em International Conference on Artificial Intelligence and
  Statistics}, 2018.

\bibitem{mansouri2019preference}
Farnam Mansouri, Yuxin Chen, Ara Vartanian, Jerry Zhu, and Adish Singla.
\newblock Preference-based batch and sequential teaching: Towards a unified
  view of models.
\newblock {\em Advances in neural information processing systems}, 32, 2019.

\bibitem{ohannessianoptimal}
Xuezhou Zhang Hrag~Gorune Ohannessian, Ayon Sen, Scott Alfeld, and Xiaojin Zhu.
\newblock Optimal teaching for online perceptrons.
\newblock {\em University of Wisconsin-Madison}, 2016.

\bibitem{qian2022teaching}
Hong Qian, Xu-Hui Liu, Chen-Xi Su, Aimin Zhou, and Yang Yu.
\newblock The teaching dimension of regularized kernel learners.
\newblock In {\em International Conference on Machine Learning}, pages
  17984--18002. PMLR, 2022.

\bibitem{rakhsha2020policy}
Amin Rakhsha, Goran Radanovic, Rati Devidze, Xiaojin Zhu, and Adish Singla.
\newblock Policy teaching via environment poisoning: Training-time adversarial
  attacks against reinforcement learning.
\newblock In {\em International Conference on Machine Learning}, pages
  7974--7984. PMLR, 2020.

\bibitem{rakhsha2021policy}
Amin Rakhsha, Goran Radanovic, Rati Devidze, Xiaojin Zhu, and Adish Singla.
\newblock Policy teaching in reinforcement learning via environment poisoning
  attacks.
\newblock {\em The Journal of Machine Learning Research}, 22(1):9567--9611,
  2021.

\bibitem{DBLP:journals/corr/abs-1905-13566}
Faraz Torabi, Garrett Warnell, and Peter Stone.
\newblock Recent advances in imitation learning from observation.
\newblock {\em CoRR}, abs/1905.13566, 2019.

\bibitem{DBLP:conf/nips/TschiatschekGHD19}
Sebastian Tschiatschek, Ahana Ghosh, Luis Haug, Rati Devidze, and Adish Singla.
\newblock Learner-aware teaching: Inverse reinforcement learning with
  preferences and constraints.
\newblock In {\em Advances in Neural Information Processing Systems}, 2019.

\bibitem{walsh2012dynamic}
Thomas~J Walsh and Sergiu Goschin.
\newblock Dynamic teaching in sequential decision making environments.
\newblock {\em arXiv preprint arXiv:1210.4918}, 2012.

\bibitem{setCover2005}
Valika~K. Wan and Khanh~Do Ba.
\newblock Approximating set cover, 2005.
\newblock Accessed: 2023-10-15.

\bibitem{wang2021teaching}
Chaoqi Wang, Adish Singla, and Yuxin Chen.
\newblock Teaching an active learner with contrastive examples.
\newblock {\em Advances in Neural Information Processing Systems},
  34:17968--17980, 2021.

\bibitem{zhang2021sample}
Xuezhou Zhang, Shubham Bharti, Yuzhe Ma, Adish Singla, and Xiaojin Zhu.
\newblock The sample complexity of teaching by reinforcement on q-learning.
\newblock In {\em Proceedings of the AAAI Conference on Artificial
  Intelligence}, volume~35, pages 10939--10947, 2021.

\bibitem{zhu2013machine}
Jerry Zhu.
\newblock Machine teaching for bayesian learners in the exponential family.
\newblock {\em Advances in Neural Information Processing Systems}, 26, 2013.

\bibitem{zhu2015machine}
Xiaojin Zhu.
\newblock Machine teaching: An inverse problem to machine learning and an
  approach toward optimal education.
\newblock In {\em Proceedings of the AAAI Conference on Artificial
  Intelligence}, volume~29, 2015.

\end{thebibliography}
